\documentclass{article}
\usepackage{ijcai20}

\usepackage{float}
\usepackage{mathpazo}
\usepackage{amssymb}
\usepackage{mathrsfs}
\usepackage{mathtools}
\usepackage{amsfonts}
\usepackage{array}
\usepackage{enumitem}

\usepackage{tikz}

\usepackage{times}
\usepackage{soul}
\usepackage{url}
\usepackage[hidelinks]{hyperref}
\usepackage[utf8]{inputenc}
\usepackage[small]{caption}
\usepackage{graphicx}
\usepackage{amsmath}
\usepackage{amsthm}
\usepackage{booktabs}
\usepackage{algorithm}
\usepackage{algorithmic}
\usepackage{multirow}
\newcommand{\ZO}{\{0,1\}}
\newcommand{\Disc}[2]{\textup{Disc}\left(#1,#2\right)}
\newcommand{\inv}[1]{#1^{\text{ }\scalebox{0.8}[0.75]{\textup{-1}}}}
\newcommand{\Ex}[1]{\textup{E}\left[#1\right]}
\newcommand{\rk}[1]{\textup{rk}(#1)}
\newcommand{\Exl}[2]{\textup{E}_{#1}\left[#2\right]}
\newcommand{\Prb}[1]{\textup{Pr}\left[#1\right]}
\newcommand{\Prbl}[2]{\textup{Pr}_{#1}\left[#2\right]}
\newcommand{\Prbu}[2]{\mathop{\textup{Pr}}_{#1}\left[#2\right]}
\newcommand{\card}[1]{\vert #1 \vert}
\newcommand{\bcard}[1]{\big\vert #1 \big\vert}

\newcommand{\dDNNF}{\textsf{d-DNNF}}
\newcommand{\OBDD}{\textsf{OBDD}}
\newcommand{\NNF}{\textsf{NNF}}
\newcommand{\DNNF}{\textsf{DNNF}}
\newcommand{\CNF}{\textsf{CNF}}
\newcommand{\AND}{\textsf{AND}}
\newcommand{\OR}{\textsf{OR}}
\newcommand{\SDD}{\textsf{SDD}}
\newcommand{\TAN}{\textsf{TAN}}

\newtheorem{theorem}{Theorem}
\newtheorem{lemma}{Lemma}
\newtheorem{definition}{Definition}
\newtheorem{claim}{Claim}

\DeclareMathSymbol{\shortminus}{\mathbin}{AMSa}{"39}
\newcommand{\overA}{\overline{\mkern-1mu A}}
\newcommand{\overB}{\overline{\mkern+0.5mu B}}

\usepackage[titletoc]{appendix}

\title{Lower Bounds for Approximate Knowledge Compilation}

\author{
Alexis de Colnet\And
Stefan Mengel\\
\affiliations
CRIL, CNRS \& Univ Artois\\
\emails
decolnet@cril.fr,
mengel@cril.fr
}

\begin{document}

\maketitle

\begin{abstract}
Knowledge compilation studies the trade-off between succinctness and efficiency of different representation languages. For many languages, there are known strong lower bounds on the representation size, but recent work shows that, for some languages, one can bypass these bounds using approximate compilation.  The idea is to compile an approximation of the knowledge for which the number of errors can be controlled. We focus on circuits in deterministic decomposable negation normal form ({\dDNNF}), a compilation language  suitable in contexts such as probabilistic reasoning, as it supports efficient model counting and probabilistic inference. Moreover, there are known size lower bounds for {\dDNNF} which by relaxing to approximation one might be able to avoid. 
In this paper we formalize two notions of approximation: \emph{weak approximation} which has been studied before in the decision diagram literature and \emph{strong approximation} which has been used in recent algorithmic results. We then show lower bounds for approximation by {\dDNNF}, complementing the positive results from the literature.
\end{abstract}

\section{Introduction}
Knowledge compilation is a subarea of artificial intelligence which studies different representations for knowledge~\cite{DarwicheM02}. The basic idea is that different types of representation are more useful when solving reasoning problems than others. One general observation that has been made is that often representations that allow many reasoning tasks to be solved efficiently, such as the classical {\OBDD}s, are necessarily large in size whereas more succinct representations often make reasoning hard. This trade-off between succinctness and usefulness is studied systematically in knowledge compilation.

One canonical area where the representation languages introduced in knowledge compilation are applied is probabilistic reasoning. For example, one can translate or \emph{compile} classifiers based on graphical models, e.g.~Bayesian networks, into such representations and then reason about the classifiers by querying the compiled representation~\cite{ChanD03}. If the representation is of reasonable size and can be computed efficiently, then the overall reasoning process is efficient. The most important representation language in this setting are circuits in deterministic, decomposable negation normal form, short {\dDNNF}, which allow for efficient (weighted) model counting and probability computation and are thus particularly well suited for probabilistic reasoning~\cite{Darwiche01c}. {\dDNNF}s are generalizations of other important languages like {\OBDD}~\cite{Bryant86} and {\SDD}~\cite{Darwiche11} which have also found applications in probabilistic reasoning~\cite{ChanD03,ChoiKD13,Shih19}. Due to their importance, essentially all practical implementations of knowledge compilers create {\dDNNF}s or sub-classes thereof~\cite{Darwiche04,MuiseMBH12,Darwiche11,OztokD15,LagniezM17}. For these reasons we focus on {\dDNNF}s in this paper.

Unfortunately, in general, representations of knowledge in {\dDNNF} are large. This had been known under standard complexity theoretical assumptions for a long time~\cite{DarwicheM02} and more recently there has been a series of papers showing exponential, unconditional lower bounds for many representation languages~\cite{BovaCMS16,Beame0RS17,PipatsrisawatD10,Capelli17,BeameL15}. Moreover, \cite{BovaCMS16} gave an explicit connection between {\DNNF} lower bounds and communication complexity, a subarea of theoretical computer science. This makes it possible to use known results from communication complexity to get strong unconditional lower bounds in knowledge compilation. As one consequence, it is now known that the representation of many problems in {\dDNNF} is infeasible.

Fortunately, this bad news is not necessarily a fatal problem for probabilistic reasoning. Since graphical models like Bayesian networks are almost exclusively inferred by learning processes, they are inherently not exact representations of the world.  Thus, when reasoning about them, in most cases the results do no have to be exact but approximate reasoning is sufficient, assuming that the approximation error can be controlled and is small. It is thus natural in this context to consider \emph{approximate knowledge compilation}: the aim is no longer to represent knowledge exactly as one authorizes a small number of errors. Very recently, Chubarian and Turán~\shortcite{ChubarianT16} showed, building on~\cite{GopalanKMSVV11}, that this approach is feasible in some settings: it is possible to compile approximations of so-called Tree Augmented Naive Bayes classifiers ({\TAN}) (or more generally bounded pathwidth Bayes classifiers) into {\OBDD}s efficiently. Note that efficient exact compilation is ruled out in this setting due to strong lower bounds for threshold functions from~\cite{TakenagaNY97} which imply lower bounds for {\TAN}s. 

In this paper, we complement the positive results of~\cite{ChubarianT16} by extending lower bounds for exact representations to lower bounds for approximations.
Similar questions have been treated before for {\OBDD}s and some extensions such as \emph{read-$k$ branching programs}, see e.g.~\cite{KrauseSW99,BolligSW02}. We extend this line of work in two ways: we show that the techniques used in~\cite{BolligSW02} can be adapted to show lower bounds for the approximation by {\dDNNF}s and prove that there are functions for which any {\dDNNF} computing a non-trivial approximation must have exponential size.

As a second contribution, we refine the  approximation notion used in~\cite{BolligSW02} which we call \emph{weak approximation}. For this notion, the approximation quality is measured as the probability of encountering an error when comparing a function and its approximation on a random input. It follows that all families of Boolean functions for which the probability of encountering a model on a random input is not bounded by a constant, can be approximated trivially by constant functions (see Section~\ref{sec:strong_approximation} for details). This makes weak approximation easy for rather uninteresting reasons for many functions, e.g.~most functions given by {\CNF}-formulas. Moreover, it makes the approximation quality sensitive to encodings, in particular the use of auxiliary variables that functionally depend on the input. In general, the space of satisfying assignments is arguably badly described by weak approximations. In particular, the relative error for model counting and probability evaluation is unbounded which makes that notion useless for probabilistic reasoning.

We remedy the situation by formalizing a new notion of approximation for knowledge compilation which we call \emph{strong approximation}. It is modeled to allow efficient counting with approximation guarantees and is insensitive to addition of functionally dependent auxiliary variables, see Section~\ref{sec:strong_approximation} for the definition and detailed discussion. While not formalized as such, it can be verified that the {\OBDD}s of~\cite{ChubarianT16,GopalanKMSVV11} are in fact strong approximations in our sense. We then show that weak and strong approximations differ by exhibiting a family of functions that has trivial weak approximations but any {\dDNNF}s approximating it non-trivially must be of exponential size.

We remark that approximation in knowledge compilation has been considered before -- in fact one of the earliest lines of work in the setting was approximating Boolean functions by Horn formulas~\cite{SelmanK96}. However, the focus was different in this setting: on the one hand, Horn formulas are not fully expressive so the question becomes that of understanding the formulas that are the best out of all Horn formulas approximating a function instead of requesting error guarantees for the approximation. On the other hand, that line of work was less concerned with the size of the approximating formulas and more with their existence.
Our work is different in these respects: since we deal with a fully expressive representation language, the main concern becomes that of a trade-off between the quality of approximation (measured in the number of inputs in which the function at hand and its approximation differ) and the representation size of the approximation.

\paragraph{Outline of the paper.} We give some preliminaries in Section~\ref{sec:preliminaries}. We then introduce the notion of weak approximation and show our lower bound for it in Section~\ref{sec:weak_approximation}. We introduce and discuss strong approximations next in Section~\ref{sec:strong_approximation} and show that weak and strong approximations differ in Section~\ref{sec:main_result}. We close the paper with some conclusions and open questions in Section~\ref{sec:conclusion}. Due to space constraints some of the proofs are not contained in this version of the paper and will appear in the upcoming full version.

\section{Preliminaries}
\label{sec:preliminaries}
We describe some conventions of notation for Boolean algebra. In our framework, a Boolean variable takes value $0$ ($false$) or $1$ ($true$), we see it as a variable over $\mathbb{F}_2$, the field with two elements. Assignments of $n$ Boolean variables are vectors from $\mathbb{F}^n_2$ and operations on vectors and matrices are considered in this field. We use the notation $\mathbf{0}^n$ to denote the $0$-vector from $\mathbb{F}^n_2$. For clarity we also use the operators $\neg$, $\vee$ and~$\wedge$ for negation, disjunction and conjunction in $\mathbb{F}_2$. The conjunction of Boolean variables and the product in $\mathbb{F}_2$ are equivalent and used interchangeably.
Single variables are written in plain style ``$x$'' while assignments of $n > 1$ variables use bold style ``$\mathbf{x}$''. A Boolean function on $n$ variables is a mapping $f : \mathbb{F}^n_2 \rightarrow \mathbb{F}_2$ and its models are given by $\inv{f}(1)$. Given a set of assignments $S$, we sometimes denote $\mathbb{1}_S$ the Boolean function whose set of models is exactly $S$. We write $f \leq g$ when $\inv{f}(1) \subseteq \inv{g}(1)$, which corresponds to logical entailment. A distribution on truth assignments is a probabilistic distribution $\mathcal{D}$ on $\mathbb{F}^n_2$. We write $\Prbl{\mathbf{x} \sim \mathcal{D}}{\cdot}$ to denote the probability measure when sampling an assignment $\mathbf{x}$ according to~$\mathcal{D}$. For clarity, the uniform distribution on $\mathbb{F}^n_2$ is denoted $\mathcal{U}$ (regardless of $n$), $\mathbf{x} \sim \mathcal{U}$ means that any assignment is sampled with probability $1/2^n$.

\paragraph{Deterministic decomposable NNF.} Let $X$ be a finite set of Boolean variables. A circuit in \emph{negation normal form}  ({\NNF}) over $X$ is a single output Boolean circuit whose inputs gates are labeled with Boolean variables $x$ from $X$ and their negations $\neg x$ and whose internal gates are fanin-2 {\AND} and {\OR}-gates. The \emph{size} of a circuit is the number of its gates. A circuit over $X$ is said to \emph{accept} a truth assignment $\mathbf{x}$ of the variables if it outputs 1 ($true$) when its inputs are set as in $\mathbf{x}$. In this case $\mathbf{x}$ is a \emph{model} of the function represented by the circuit. An {\NNF} is \emph{decomposable} if, for any {\AND}-gate $g$, the two sub-circuits rooted at $g$ share no input variable, i.e., if $x$ or $\neg x$ is an input gate of the circuit rooted at the left input of $g$, then neither $x$ nor $\neg x$ is an input gate of the subcircuit rooted at the right input, and vice versa. An {\NNF} is \emph{deterministic} if, for any {\OR}-gate $g$, the sets of assignments accepted by the two subcircuits rooted at the children of $g$ are disjoint. A decomposable {\NNF} is called a {\DNNF}; if in addition it is deterministic, then it is called a {\dDNNF}.

\paragraph{Rectangle covers.} Let $X$ be a finite set of Boolean variables. A \emph{combinatorial rectangle} over $X$ (more succinctly a \emph{rectangle}) is a Boolean function $r$ defined as the conjunction of two Boolean functions $\rho_1$ and $\rho_2$ over disjoints variables of $X$. That is, there is a partition $(X_1, X_2)$ of $X$ such that $\rho_1 $ and $\rho_2$ are defined over $X_1$ and $X_2$, respectively, and $r = \rho_1 \wedge \rho_2$. We call $(X_1, X_2)$ the \emph{partition} of $r$. The rectangle is \emph{balanced} if $\card{X}/3 \leq \card{X_1} \leq 2\card{X}/3$ (the same bounds hold for $\card{X_2}$). A \emph{rectangle cover} of a Boolean function $f$ is any disjunction of rectangles over $X$ (possibly for different partitions of $X$) equivalent to $f$, i.e., $f=\bigvee_{i=1}^K r_i$ where the $r_i$ are rectangles. The \emph{size} of a cover is the number $K$ of its rectangles. A rectangle cover is called \emph{balanced} if its rectangles are balanced and it is said \emph{disjoint} if no two rectangles share a model. Note that any function $f$ has at least one balanced disjoint rectangle cover, because it can be written as a \textsf{DNF} in which every term contains all variables. There is a tight link between the smallest size of a balanced disjoint rectangle cover of a function and the size of any equivalent {\dDNNF}.

\begin{theorem}\textup{\cite{BovaCMS16}}\label{theorem:DNNF_size_rect_cover_size}
Let $D$ be a {\dDNNF} encoding a function $f$. Then $f$ has a balanced disjoint rectangle cover of size at most the size of $D$.
\end{theorem}

\noindent Theorem~\ref{theorem:DNNF_size_rect_cover_size} implies that, to show a lower bound on the size of any {\dDNNF} encoding $f$, it is sufficient to find a lower bound on the size of any balanced disjoint rectangle cover of $f$.

\section{Large d-DNNFs for Weak Approximations}
\label{sec:weak_approximation}
In this section, we start by considering the notion of approximation that has been studied for different forms of branching programs before, see e.g.~\cite{KrauseSW99,BolligSW02}. To differentiate it from other notions, we give it the name \emph{weak approximation}.

\begin{definition}[Weak approximation]\label{definition:weak_approximation}
Let $\mathcal{D}$ be a distribution on the truth assignments to $X$ and $\varepsilon > 0$. We say that~$\tilde{f}$ is a \emph{weak $\varepsilon$-approximation} of~$f$ (or weakly $\varepsilon$-approximates $f$) with respect to $\mathcal{D}$ if 
$$
\Prbu{\mathbf{x} \sim \mathcal{D}}{f(\mathbf{x}) \neq \tilde{f}(\mathbf{x})} \leq \varepsilon.
$$
\end{definition}

\noindent When $\mathcal{D}$ is the uniform distribution $\mathcal{U}$, then the condition of weak $\varepsilon$-approximability is equivalent to $\card{\{ \mathbf{x} : f(\mathbf{x}) \neq \tilde{f}(\mathbf{x}) \}} \leq \varepsilon 2^n$. 
\par
Note that weak $\varepsilon$-approximation is only useful when $\varepsilon < 1/2$. This is because every function has a trivial $(1/2)$-approximation: if $\Prbu{\mathbf{x} \sim \mathcal{D}}{f(\mathbf{x}) = 1}>1/2$, then the constant $1$-function is a $(1/2)$-approximation, otherwise this is the case for the constant $0$-function. Note that it might be hard to decide which case is true, but in any case we know that the approximation ratio of one of the constants is good.
\par
Bollig \emph{et al.}~\shortcite{BolligSW02} used a \emph{discrepancy} argument to show that there are classes of functions such that any $\varepsilon$-approximation w.r.t.~$\mathcal{U}$ requires exponential {\OBDD} size. We lift their techniques to {\dDNNF} showing that the same functions are also hard for {\dDNNF}.

\begin{theorem}\label{theorem:main_result_bilinear_form}
Let $0 \leq \varepsilon < 1/2$, there is a class of Boolean functions $\mathcal{C}$ such that, for any $f \in \mathcal{C}$ on $n$ variables, any {\dDNNF} encoding a weak $\varepsilon$-approximation of $f$ w.r.t.~$\mathcal{U}$ has size $2^{\Omega(n)}$.
\end{theorem}

Since {\dDNNF}s are strictly more succinct than {\OBDD}s~\cite{DarwicheM02}, Theorem~\ref{theorem:main_result_bilinear_form} is a generalization of the result on {\OBDD}s in \cite{BolligSW02}. However, since the proof is almost identical, differing near the end only, we defer the technical details to the full version. We here only introduce the notion of discrepancy that is central to the proof and will be useful later. 

\paragraph{The discrepancy method.} We want to use Theorem~\ref{theorem:DNNF_size_rect_cover_size} to bound the size of a {\dDNNF} encoding $\tilde{f}$ a weak $\varepsilon$-approximation of $f$ w.r.t.~some distribution. To this end we study disjoint balanced rectangle covers of $\tilde{f}$. Let~$r$ be a rectangle from such a cover. $r$ can make \emph{false positives} on~$f$, i.e.,~have models that are not models of~$f$. Similarly, \emph{true positives} are models shared by~$r$ and~$f$. The \emph{discrepancy} $\Disc{f}{r}$ of $f$ on $r$ is the difference between the number of false positives and true positives, normalized by the total number of assignments: $\Disc{f}{r} := \frac{1}{2^n} \bcard{\card{\inv{r}(1) \cap \inv{f}(1)} - \card{\inv{r}(1) \cap \inv{f}(0)}}$. A small discrepancy indicates that $r$ has few models or that it makes roughly as many false positives as true positives on~$f$. Discrepancy bounds have been used before to prove results in distributional communication complexity~\cite[Chapter 3.5]{KushilevitzN97}. Here we show that when there is an upper bound on $\Disc{f}{r}$ for any rectangle $r$ from a cover of $\tilde{f}$, one can obtain a lower bound on the size of the cover of~$\tilde{f}$.

\begin{lemma}\label{lemma:bound_size_cover_weak_approximations}
Let $f$ be a Boolean function on $n$ variables and let $\tilde{f}$ be a weak $\varepsilon$-approximation of $f$ w.r.t.~$\mathcal{U}$. Let $\tilde{f} = \bigvee_{k = 1}^K r_k$ be a disjoint balanced rectangle cover of $\tilde{f}$ and assume that there is an integer $\Delta > 0$ such that $\Disc{f}{r_k} \leq \Delta/2^n$ for for all $r_k$. Then $K \geq (\card{\inv{f}(1)} - \varepsilon 2^n)/\Delta$.
\end{lemma}
\begin{proof} We have $\card{f \neq \tilde{f}} = \card{\{\mathbf{x} : f(\mathbf{x}) \neq \tilde{f}(\mathbf{x})\}}$ 
\begin{equation*}
\begin{aligned}
&= \card{\inv{f}(1) \cap \inv{\tilde f}(0) } + \card{\inv{f}(0) \cap \inv{\tilde f}(1)} 
\\
&= \big\vert \inv{f}(1) \cap \bigcap\nolimits_{k = 1}^K\inv{r_k}(0) \big\vert + \big\vert \inv{f}(0) \cap \bigcup\nolimits_{k = 1}^K \inv{r_k}(1) \big\vert  
\\
&= \card{\inv{f}(1)} - \sum_{k = 1}^K (\card{\inv{r_k}(1)  \cap \inv{f}(1)} - \card{\inv{r_k}(1) \cap \inv{f}(0)}) 
\\
&\geq \card{\inv{f}(1)} - 2^n \sum\nolimits_{k = 1}^K \Disc{f}{r_k} \geq \card{\inv{f}(1)} - K \Delta
\end{aligned}
\end{equation*}
where the last equality is due to the rectangles being disjoint. The weak $\varepsilon$-approximation w.r.t.~the uniform distribution $\mathcal{U}$ gives that $\card{\tilde{f} \neq f} \leq \varepsilon 2^n$, which we use to conclude.  
\end{proof}

\noindent Combining Lemma~\ref{lemma:bound_size_cover_weak_approximations} with Theorem~\ref{theorem:DNNF_size_rect_cover_size}, the proof of Theorem~\ref{theorem:main_result_bilinear_form} boils down to showing that there are functions such that for every balanced rectangle $r$, the discrepancy $\Disc{f}{r}$ can be suitably bounded, as shown in \cite{BolligSW02}.

\section{Strong Approximations}
\label{sec:strong_approximation}
In this section, we discuss some shortcomings of weak approximation and propose a stronger notion of approximation that avoids them. Let $f_0$ be the constant $0$-function. We say that a function is \emph{trivially weakly $\varepsilon$-approximable} (w.r.t.~some distribution) if $f_0$ is a weak $\varepsilon$-approximation. Considering approximations w.r.t. the uniform distribution, it is easy to find classes of functions that are trivially weakly approximable.

\begin{lemma}\label{lemma:limitations_weak_approx}
Let $\varepsilon > 0$ and $0 \leq \alpha < 1$. Let $\mathcal{C}$ be a class of functions such that every function in $\mathcal{C}$ on $n$ variables has at most $2^{\alpha n}$ models. Then there exists a constant $n_0$, such that any function from $\mathcal{C}$ on more than $n_0$ variables is trivially weakly $\varepsilon$-approximable w.r.t.~the uniform distribution. 
\begin{proof}
Take $n_0 = \frac{1}{1-\alpha}\log(\frac{1}{\varepsilon})$ and choose $f$ any function from~$\mathcal{C}$ on $n > n_0$ variables. Then $\card{\{\mathbf{x} : f(\mathbf{x}) \neq f_0(\mathbf{x})\}} = \card{\inv{f}(1)} \leq 2^{\alpha n} < \varepsilon 2^n$. Therefore $f_0$ is a weak $\varepsilon$-approximation (w.r.t.~the uniform distribution) of any function of $\mathcal{C}$ on sufficiently many variables.
\end{proof}
\end{lemma}

We remark that similar trivial approximation results can be shown for other distributions if the probability of a random assignment w.r.t.~this distribution being a model is very small.

As a consequence, weak approximation makes no sense for functions with ``few'' (or ``improbable'') models. However such functions are often encountered, for example, random $k$-{\CNF} with sufficiently many clauses are expected to have few models. Furthermore, even for functions with ``many'' models, one often studies encodings over larger sets of variables. For instance, when using Tseitin encoding to transform Boolean circuits into {\CNF}, one introduces auxiliary variables that compute the value of sub-circuits under a given assignment. Generally, auxiliary variables are often used in practice since they reduce the representation size of functions, see e.g.~\cite[Chapter~2]{Handbook09}. The resulting encodings have more variables but most of the time the same number of models as the initial function. Consequently, they are likely to be trivially weakly approximable from Lemma~\ref{lemma:limitations_weak_approx}. For these reasons we define a stronger notion of approximation.

\begin{definition}[Strong approximation]\label{definition:strong_approximation}
Let $\mathcal{D}$ be a distribution of the truth assignments to $X$ and $\varepsilon > 0$. We say that $\tilde{f}$ is a \emph{strong $\varepsilon$-approximation} of $f$ (or strongly $\varepsilon$-approximates $f$) with respect to $\mathcal{D}$ if 
$$
\Prbu{\mathbf{x} \sim \mathcal{D}}{f(\mathbf{x}) \neq \tilde{f}(\mathbf{x})} \leq \varepsilon \Prbu{\mathbf{x} \sim \mathcal{D}}{f(\mathbf{x}) = 1}.
$$
\end{definition}
\noindent When $\mathcal{D}$ is the uniform distribution $\mathcal{U}$, then the condition of strong approximability is equivalent to $\card{\{ \mathbf{x} : f(\mathbf{x}) \neq \tilde{f}(\mathbf{x}) \}} \leq \varepsilon \card{\inv{f}(1)}$. It is easy to see that strong approximation does not have the problem described in Lemma~\ref{lemma:limitations_weak_approx} for weak approximation.
We also remark that strong approximation has been modeled to allow for efficient counting. In fact, a {\dDNNF} computing a strong $\varepsilon$-approximation of a function~$f$ allows approximate model counting for~$f$ with approximation factor~$\varepsilon$.

\par Strong approximation has implicitly already been used in knowledge compilation. For instance it has been shown in \cite{GopalanKMSVV11} -- although the authors use a different terminology -- that for $\varepsilon > 0$, any Knapsack functions on $n$ variables has a strong $\varepsilon$-approximation w.r.t.~$\mathcal{U}$ that can be encoded by an {\OBDD} of size polynomial in $n$ and $1/\varepsilon$. The generalization to {\TAN}s~\cite{ChubarianT16} is also for strong approximations. These results are all the more significant since we know from~\cite{TakenagaNY97} that there exist threshold functions for which exact encodings by {\OBDD} require size exponential in $n$.
\par Obviously, a strong approximation of $f$ w.r.t.~some distribution is also a weak approximation. Thus the statement of Theorem~\ref{theorem:main_result_bilinear_form} can trivially be lifted to strong approximation. However the hard functions from Theorem~\ref{theorem:main_result_bilinear_form} necessarily have sufficiently many models: if we are to consider only functions with few models, then they all are trivially weakly approximable. Yet we prove in the next section that there exist such functions whose exact encoding and strong $\varepsilon$-approximation encodings by {\dDNNF} require size exponential in $n$. Our proof follows the discrepancy method but relies on the following variant of Lemma~\ref{lemma:bound_size_cover_weak_approximations} for strong approximation.

\begin{lemma}\label{lemma:bound_size_cover_strong_approximations}
Let $f$ be a Boolean function on $n$ variables and let $\tilde{f}$ be a strong $\varepsilon$-approximation of $f$ w.r.t.~$\mathcal{U}$. Let $\tilde{f} = \bigvee_{k = 1}^K r_k$ be a disjoint balanced rectangle cover of $\tilde{f}$ and assume that there is an integer $\Delta > 0$ such that $\Disc{f}{r_k} \leq \Delta/2^n$ for for all $r_k$. Then $K \geq (1-\varepsilon)\card{\inv{f}(1)}/\Delta$.
\begin{proof}
The proof is essentially the same as for Lemma~\ref{lemma:bound_size_cover_weak_approximations}, differing only in the last lines where we use $\card{\tilde{f} \neq f} \leq \varepsilon \card{\inv{f}(1)}$ rather than $\card{\tilde{f} \neq f} \leq \varepsilon 2^n$.
\end{proof}
\end{lemma}

\section{Large d-DNNFs for Strong Approximations}
\label{sec:main_result}
In this section, we show a lower bound for strong approximations of some functions that have weak approximations by small {\dDNNF}s.
The functions we consider are characteristic functions of linear codes which we introduce now: a \emph{linear code} of length $n$ is a linear subspace of the vector space $\mathbb{F}^n_2$. Vectors from this subspace are called \emph{code words}. A linear code is characterized by a \emph{parity check matrix} $H$ from $\mathbb{F}^{m \times n}_2$ as follows: a vector $\mathbf{x} \in \mathbb{F}^n_2$ is a code word if and only if $H\mathbf{x} = \mathbf{0}^m$ (operations are modulo 2 in $\mathbb{F}^n_2$). The \emph{characteristic function} of a linear code is a Boolean function which accepts exactly the code words. Note that the characteristic function of a length $n$ linear code of check matrix $H$ has $2^{n-\rk{H}}$ models, where $\rk{H}$ denotes the rank of~$H$. Following ideas developed in \cite{Duris04}, we focus on linear codes whose check matrices~$H$ have the following property: $H$ is called \emph{$s$-good} for some integer $s$ if any submatrix obtained by taking at least $n/3$ columns\footnote{Duris \emph{et al.}~\shortcite{Duris04} limit to submatrices built from at least $n/2$ columns rather than $n/3$; however their result can easily be adapted.} from~$H$ has rank at least $s$. The existence of $s$-good matrices for $s = m-1$ is guaranteed by the next lemma.

\begin{lemma}\textup{\cite{Duris04}}\label{lemma:good_matrices}
Let $m = n/100$ and sample a parity check matrix $H$ uniformly at random from $\mathbb{F}^{m \times n}_2$. Then $H$ is $(m-1)$-good with probability $1 - 2^{-\Omega(n)}$.
\end{lemma}

Our interest in linear codes characterized by $s$-good matrices is motivated by another result from \cite{Duris04} which states that the maximal size of any rectangle entailing the characteristic function of such a code decreases as $s$ increases.

\begin{lemma}\textup{\cite{Duris04}}\label{lemma:size_perfect_rectangle}
Let $f$ be the characteristic function of a linear code of length $n$ characterized by the $s$-good matrix $H$. Let $r$ be a balanced rectangle such that $r \leq$~$f$. Then $\card{\inv{r}(1)} \leq 2^{n-2s}$.
\end{lemma}

Combining Lemmas~\ref{lemma:good_matrices} and~\ref{lemma:size_perfect_rectangle} with Theorem~\ref{theorem:DNNF_size_rect_cover_size}, one gets the following result that was already observed in~\cite{Mengel16}:

\begin{theorem}\label{theorem:large_dDNNF_linear_code}
There exists a class of linear codes $\mathcal{C}$ such that, for any code from $\mathcal{C}$ of length $n$, any {\dDNNF} encoding its characteristic function has size $2^{\Omega(n)}$.
\end{theorem}

In the following, we will show that not only are characteristic functions hard to represent exactly as {\dDNNF}, they are even hard to strongly approximate.

Given the characteristic function $f$ of a length $n$ linear code of check matrix $H$, $f$ has exactly $2^{n - \rk{H}}$ models. When $\rk{H}$ is at least a constant fraction of $n$, $f$ satisfies the condition of Lemma~\ref{lemma:limitations_weak_approx}, so for every $\varepsilon > 0$ and $n$ large enough, $f$ is trivially weakly $\varepsilon$-approximable (w.r.t.~the uniform distribution). However we will show that
any strong $\varepsilon$-approximation $\tilde{f}$ of $f$ (w.r.t.~the uniform distribution) only has {\dDNNF} encodings of size exponential in $n$. 

To show this result, we will use the discrepancy method: we are going to find a bound on the discrepancy of $f$ on any rectangle from a balanced disjoint rectangle cover of $\tilde{f}$. Then we will use the bound in Lemma~\ref{lemma:bound_size_cover_strong_approximations} and combine the result with Theorem~\ref{theorem:DNNF_size_rect_cover_size} to finish the proof. 

Note that it is possible that a rectangle from a disjoint rectangle cover of $\tilde{f}$ makes no false positives on $f$. In fact, if this is the case for all rectangles in the cover, then $\tilde{f} \leq f$. In this case, lower bounds can be shown essentially as in the proof of Theorem~\ref{theorem:large_dDNNF_linear_code}. The more interesting case is thus that in which rectangles make false positives. In this case, we assume that no rectangle makes more false positives on $f$ than it accepts models of $f$, because if such a rectangle $r$ exists in a disjoint cover of $\tilde{f}$, then deleting $r$ leads to a better approximation of $f$ than $\tilde{f}$. Thus it is sufficient to consider approximations and rectangle covers in which all rectangles verify $\card{\inv{r}(1) \cap \inv{f}(1)} \geq \card{\inv{r}(1) \cap \inv{f}(0)}$.

\begin{definition}\label{definition:core_rectangle}
Let $r$ be a rectangle. A \emph{core rectangle} (more succinctly a \emph{core}) of $r$ w.r.t.~$f$ is a rectangle $r_{\text{core}}$ with the same partition as $r$ such that 
\begin{enumerate}
\item[a)] $r_\text{core} \leq f$ and $r_\text{core} \leq r$,
\item[b)] $r_{core}$ is maximal in the sense that there is no $r'$ satisfying a) such that $\card{\inv{r'}(1)} > \card{\inv{r_{core}}(1)}$.
\end{enumerate}
\end{definition}

\noindent
Note that if $r \leq f$, then the only core rectangle of $r$ is $r$ itself. Otherwise $r$ may have several core rectangles. We next state a crucial lemma on the relation of discrepancy and cores whose proof we defer to later parts of this section.

\begin{lemma}\label{lemma:discr}
Let $f$ be the characteristic function of some length $n$ linear code, let $r$ be a rectangle with more true positives than false positives on $f$, and let $r_\text{core}$ be a core rectangle of~$r$ with respect to $f$, then 
$$
\Disc{f}{r} \leq \frac{1}{2^n}\card{\inv{r_{\text{core}}}(1)}.
$$
\end{lemma}

Lemma~\ref{lemma:discr} says the following: consider a rectangle $r_{\text{core}} \leq f$ which is a core of a rectangle $r$. If $r$ accepts more models of $f$ than $r_{\text{core}}$, then for each additional such model $r$ accepts at least one false positive. With Lemma~\ref{lemma:discr}, it is straightforward to show the main result of this section.

\begin{theorem}\label{theorem:main_result_linear_code}
Let $0 \leq \varepsilon < 1$. There is a class of Boolean functions $\mathcal{C}$ such that any $f \in \mathcal{C}$ on $n$ variables is trivially weakly $\varepsilon$-approximable w.r.t.~$\mathcal{U}$ but any {\dDNNF} encoding a strong $\varepsilon$-approximation w.r.t.~$\mathcal{U}$ has size $2^{\Omega(n)}$.
\end{theorem}
\begin{proof}
Choose $\mathcal{C}$ to be the class of characteristic functions for length $n$ linear codes characterized by $(m-1)$-good check matrices with $m = n/100$. Existence of these functions as $n$ increases is guaranteed by Lemma~\ref{lemma:good_matrices}. Let $\tilde{f}$ be a strong $\varepsilon$-approximation of $f \in \mathcal{C}$ w.r.t.~$\mathcal{U}$ and let $\bigvee_{k = 1}^K r_k$ be a rectangle cover of $\tilde{f}$. Combining Lemma~\ref{lemma:discr} with Lemma~\ref{lemma:size_perfect_rectangle}, we obtain $\Disc{f}{r_k} \leq 2^{-n}2^{n-2(m-1)}$. We then use Lemma~\ref{lemma:bound_size_cover_strong_approximations} to get $K \geq (1-\varepsilon)2^{2m-n}\card{\inv{f}(1)}/4$. The rank of the check matrix of $f$ is at most $m$ so $\card{\inv{f}(1)} \geq 2^{n-m}$ and $K \geq (1-\varepsilon)2^{m}/4 = (1-\varepsilon)2^{\Omega(n)}$. We use Theorem~\ref{theorem:DNNF_size_rect_cover_size} to conclude.
\end{proof}

Note that Theorem~\ref{theorem:main_result_linear_code} is optimal w.r.t.~$\varepsilon$ since for $\varepsilon=1$ there is always the trivial approximation by the constant $0$-function.

It remains to show Lemma~\ref{lemma:discr} in the remainder of this section to complete the proof of Theorem~\ref{theorem:main_result_linear_code}. To this end, we make another definition.

\begin{definition}\label{definition:core_extraction_operator}
 Let $(X_1, X_2)$  be a partition of the variables of~$f$. A \emph{core extraction operator} w.r.t.~$f$ is a mapping $\mathcal{C}_f$ that maps every pair $(S_1, S_2)$ of sets of assignments over $X_1$ and $X_2$, respectively, to a pair $(S_1', S_2')$ with
\begin{enumerate}
\item[a)] $S_1' \subseteq S_1$ and $S_2' \subseteq S_2$,
\item[b)] assignments from $S_1' \times S_2'$ are models of $f$,
\item[c)] if $f$ has no model in $S_1 \times S_2$, then $S_1' = S_2' = \emptyset$,
\item[d)] $S_1'$ and $S_2'$ are maximal in the sense that for every $S_1'' \subseteq S_1$ and every $S_2'' \subseteq S_2$ respecting the properties a), b) and c), we have $\card{S_1'}\card{S_2'} \geq \card{S_1''}\card{S_2''}$.
\end{enumerate}
\end{definition}

Intuitively $S'_1$ and $S'_2$ are the largest subsets one can extract from $S_1$ and $S_2$ such that assignments from $S_1' \times S_2'$ are models of $f$. Note that, similarly to rectangle cores, the sets $S_1'$ and $S_2'$ are not necessarily uniquely defined. In this case, we assume that $\mathcal{C}_f$ returns an arbitrary pair with the required properties. One can show that core extraction operators yield core rectangles, as their name suggests.

\begin{claim}\label{claim:core}
Let $r = \rho_1 \wedge \rho_2$ be a rectangle w.r.t.~the partition $(X_1, X_2)$ and denote $(A,B) = \mathcal{C}_f(\inv{\rho_1}(1), \inv{\rho_2}(1))$. Then the rectangle $\mathbb{1}_A \wedge \mathbb{1}_B$ is a core rectangle of $r$ w.r.t.~$f$.
\end{claim}
The proof of Claim~\ref{claim:core} and those of several other claims in this section are deferred to the full version due to space constraints. At this point, recall that $f$ is the characteristic function of a linear code for a $m \times n$ check matrix $H$.

\begin{claim}\label{claim:false_positives}
Let $r = \rho_1 \wedge \rho_2$ be a rectangle w.r.t.~the partition $(X_1,X_2)$. Let $(A, B) = \mathcal{C}_f(\inv{\rho_1}(1), \inv{\rho_2}(1))$ and consider the core rectangle $r_\text{core} = \mathbb{1}_A \wedge \mathbb{1}_B$. Let $\overA = \inv{\rho_1}(1) \setminus A$ and $\overB = \inv{\rho_2}(1) \setminus B$. Then all assignments from $\overA \times B$ and $A \times\overB$ are false positives of $r$ on $f$.
\end{claim}
\begin{proof}
Index the $n$ columns of $H$ with the variables in $X$ ($x_1$ for column 1, $x_2$ for column 2, and so on). Let $H_1$ (resp. $H_2$) be the matrix obtained taking only the columns of $H$ whose indices are in $X_1$ (resp. $X_2$).
Obviously all vectors in $\overA \times B$ and $A \times \overB$ are models of $r$, but we will prove that they are not models of $f$.
For every $\mathbf{a}' \in \overA$ there is $\mathbf{b} \in B$ such that $H(\mathbf{a}', \mathbf{b}) = H_1\mathbf{a}' + H_2\mathbf{b} \neq \mathbf{0}^m$, otherwise the core rectangle would not be maximal. By definition of $A$ and $B$, given $\mathbf{a} \in A$, for all $\mathbf{b} \in B$ we have $H(\mathbf{a}, \mathbf{b}) = H_1\mathbf{a} + H_2\mathbf{b} = \mathbf{0}^m$, so $H_2\mathbf{b}$ is constant over $B$. Therefore if $H_1\mathbf{a}' \neq H_2\mathbf{b}$ for \emph{some} $\mathbf{b} \in B$ then $H_1\mathbf{a}' \neq H_2\mathbf{b}$ for \emph{all} $\mathbf{b} \in B$. But then no vector from $\{\mathbf{a}'\} \times B$ can be a model of $f$ and since $\mathbf{a}'$ has been chosen arbitrarily in $\overA$, all vectors from $\overA \times B$ are false positives. The case for $A \times \overB$ follows analogously.
\end{proof}

For $A$ and $B$ defined as in Claim~\ref{claim:false_positives}, we know that the assignments from $A \times B$ are models of $f$, and that those from $\overA \times B$ and $A \times \overB$ are not, but we have yet to discuss the case of $\overA \times \overB$. There may be additional models in this last set. The key to proving Lemma~\ref{lemma:discr} is to iteratively extract core rectangles from $\mathbb{1}_{\overA} \wedge \mathbb{1}_{\overB}$ and control how many false positives are generated at each step of the iteration. To this end we define the collection $((A_i, B_i))_{i = 0}^{l+1}$ as follows:
\begin{enumerate}[leftmargin=*]
\item[$\bullet$] $A_0 = \inv{\rho_1}(1)$ and $B_0 = \inv{\rho_2}(1)$,
\item[$\bullet$] for $i \geq 1$, $(A_{i}, B_{i}) = \mathcal{C}_f(A_0 \setminus \bigcup_{j = 1}^{i-1} A_j \text{ },\text{ } B_0 \setminus \bigcup_{j = 1}^{i-1} B_j)$,
\item[$\bullet$] $A_{l+1}$ and $B_{l+1}$ are empty, but for any $i < l+1$, neither $A_i$ nor $B_i$ is empty.
\end{enumerate}
Denoting $\overA_i \coloneqq A_0 \setminus \bigcup_{j = 1}^{i} A_j$ and $\overB_i \coloneqq B_0 \setminus \bigcup_{j = 1}^{i} B_j$, we can write $(A_{i}, B_{i}) = \mathcal{C}_f(\overA_{i-1},\overB_{i-1})$ (note that $\overA_0 = A_0$ and $\overB_0 = B_0$).
Basically, we extract a core $(\mathbb{1}_{A_1} \wedge \mathbb{1}_{B_1})$ from $r$, then we extract a core $(\mathbb{1}_{A_2} \wedge \mathbb{1}_{B_2})$ from $(\mathbb{1}_{\overA_1} \wedge \mathbb{1}_{\overB_1})$, and so on until there is no model of $f$ left in $\overA_l \times \overB_l$, in which case no core can be extracted from $(\mathbb{1}_{\overA_l} \wedge \mathbb{1}_{\overB_l})$ and $\mathcal{C}_f(\overA_l, \overB_l)$ returns $(\emptyset, \emptyset)$.
The construction is illustrated in Figure~\ref{fig:iterative}.
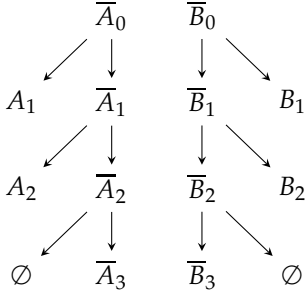
\begin{figure}[t]
\centering
\begin{tikzpicture}
\def\x{0.6}
\def\y{1.15}
	\node (A0) at (-\x,0) {$\overA_0$};
	\node (A01) at (-\x,-\y) {$\overA_1$};  \node (A1) at (-3*\x,-\y) {$A_1$};
	\node (A02) at (-\x,-2*\y) {$\overA_2$};  \node (A2) at (-3*\x,-2*\y) {$A_2$};
	\node (A03) at (-\x,-3*\y) {$\overA_3$};  \node (A3) at (-3*\x,-3*\y) {$\emptyset$};

	\draw[-stealth] (A0) -- (A01); 
	\draw[-stealth] (A0) -- (A1);
	\draw[-stealth] (A01) -- (A02); \draw[-stealth] (A01) -- (A2);
	\draw[-stealth] (A02) -- (A03); \draw[-stealth] (A02) -- (A3);
	
	\node (B0) at (+\x,0) {$\overB_0$};
	\node (B01) at (+\x,-\y) {$\overB_1$};  \node (B1) at (+3*\x,-\y) {$B_1$};
	\node (B02) at (+\x,-2*\y) {$\overB_2$};  \node (B2) at (+3*\x,-2*\y) {$B_2$};
	\node (B03) at (+\x,-3*\y) {$\overB_3$};  \node (B3) at (+3*\x,-3*\y) {$\emptyset$};

	\draw[-stealth] (B0) -- (B01); \draw[-stealth] (B0) -- (B1);
	\draw[-stealth] (B01) -- (B02); \draw[-stealth] (B01) -- (B2);
	\draw[-stealth] (B02) -- (B03); \draw[-stealth] (B02) -- (B3);
\end{tikzpicture}
\caption{An iterative core extraction with $l = 2$}
\label{fig:iterative}
\end{figure}
\begin{claim}\label{claim:rec_false_positives}
For any $i > 0$, all assignments from $F_i:= (A_i \times \overB_i) \cup (\overA_i \times B_i)$ are false positives of $r$ on $f$. Furthermore for every $i\ne j$ we have $F_i \cap F_j = \emptyset$.
\end{claim}
\begin{claim}\label{claim:rec_rect_coverage}
The function $\bigvee_{i = 1}^l (\mathbb{1}_{A_i} \wedge \mathbb{1}_{B_i})$ is a disjoint rectangle cover of $r \wedge f$. Furthermore, if $r$ is balanced, so are the rectangles from $\bigvee_{i = 1}^l (\mathbb{1}_{A_i} \wedge \mathbb{1}_{B_i})$.
\end{claim}
\noindent
With Claim~\ref{claim:rec_false_positives} and Claim~\ref{claim:rec_rect_coverage}, we can now prove Lemma~\ref{lemma:discr}.
\begin{proof}[Proof of Lemma~\ref{lemma:discr}]
Claims~\ref{claim:rec_false_positives} and~\ref{claim:rec_rect_coverage} show that $\bigcup_{i=1}^l (A_i \times B_i)$ $=  \inv{r}(1) \cap \inv{f}(1)$ and $\bigcup_{i=1}^l \left((A_i \times \overB_i) \cup (\overA_i \times B_i)\right) \subseteq \inv{r}(1) \cap \inv{f}(0)$ and that these unions are disjoint. First we focus on the models of $f$ covered by $r$.
\begin{equation*}
\begin{aligned}
\card{\inv{r}(1) \cap \inv{f}(1)} &= \sum_{i = 1}^l \card{A_i}\card{B_i}
= \card{\inv{r_\text{core}}(1)} + \sum_{i = 2}^l \card{A_i}\card{B_i}
\end{aligned}
\end{equation*}
where $r_{\text{core}} = \mathbb{1}_{A_1} \wedge \mathbb{1}_{B_1}$ is the first (therefore the largest) core rectangle extracted from $r$ w.r.t.~$f$. Now focus on the false positives of $r$ on $f$ 
\begin{equation*}
\begin{aligned}
\card{\inv{r}(1) \cap \inv{f}(0)} 
&\geq \sum\nolimits_{i = 1}^l\left(\card{A_i}\card{\overB_i} + \card{\overA_i}\card{B_i}\right) \\
&\geq \sum\nolimits_{i = 1}^l\left(\card{A_i}\card{B_{i+1}} + \card{A_{i+1}}\card{B_i}\right)
\end{aligned}
\end{equation*}
The maximality property of $\mathcal{C}_f$ implies $\card{A_i}\card{B_i} \geq$ $ \card{A_{i+1}}\card{B_{i+1}}$, and it follows that $\card{A_i}\card{B_{i+1}} + \card{A_{i+1}}\card{B_i} \geq \card{A_{i+1}}\card{B_{i+1}}$. Thus
\begin{equation*}
\begin{aligned}
\card{\inv{r}(1) \cap \inv{f}(0)} \geq \card{\inv{r}(1) \cap \inv{f}(1)} - \card{\inv{r_{\text{core}}}(1)}.
\end{aligned}
\end{equation*}
By assumption, $r$ accepts more models of $f$ than false positives so $\Disc{f}{r} = (\card{\inv{r}(1) \cap \inv{f}(1)} - \card{\inv{r}(1) \cap \inv{f}(0)})/2^n$ and the lemma follows directly.
\end{proof}

\section{Conclusion}
\label{sec:conclusion}
In this paper, we have formalized and studied weak and strong approximation in knowledge compilation and shown functions that are hard to approximate by {\dDNNF}s with respect to these two notions. In particular, we have shown that strong approximations by {\dDNNF}s generally require exponentially bigger {\dDNNF} representations than weak approximations.

Let us sketch some directions for future research. One obvious question is to find for which classes of functions there \emph{are} efficient algorithms computing approximations by {\dDNNF}s. In~\cite{ChubarianT16}, it is shown that this is the case for certain Bayesian networks. It would be interesting to extend this to other settings to make approximation more applicable in knowledge compilation. Of particular interest are in our opinion settings in which models are typically learned from data and thus inherently inexact, e.g.~other forms of graphical models and neural networks.

Another question is defining and analyzing more approximation notions beyond weak and strong approximation. In fact, the latter was designed to allow approximate (weighted) counting as needed in probabilistic reasoning. Are there ways of defining notions of approximation that are useful for other problems, say optimization or entailment queries?

A more technical question is if one can show lower bounds for non-deterministic {\DNNF}s. In that setting, different rectangles may share the same false positives in which case our lower bound techniques break down. Are there approaches to avoid this problem?
\section*{Acknowledgments}
This work has been partly supported by the PING/ACK project of the French National Agency for Research (ANR-18-CE40-0011).

\bibliographystyle{named}
\bibliography{main}

\section*{Appendices}
\subsection*{Missing Proofs of Section~\ref{sec:main_result}}
\label{app:codes}

Proofs for Lemmas~\ref{lemma:good_matrices} and~\ref{lemma:size_perfect_rectangle} can be found in~\cite{Duris04}. For the reader convenience they are included in this appendix almost as is. 

\setcounter{lemma}{3}

\begin{lemma}\textup{\cite{Duris04}}\label{lemma:good_matrices_appendix}
Let $m = n/100$ and sample a parity check matrix $H$ uniformly at random from $\mathbb{F}^{m \times n}_2$. Then $H$ is $(m-1)$-good with probability $1 - 2^{-\Omega(n)}$.
\begin{proof}
Let $\mathbf{c}_1$, $\dots$, $\mathbf{c}_n$ be the column vectors of $H$. Sampling $H$ uniformly at random from $\mathbb{F}^{m \times n}_2$ is equivalent to sampling the $\mathbf{c}_i$s independently and uniformly at random from $\mathbb{F}^m_2$. 
\par
By definition, $H$ is $(m-1)$-good if and only if any submatrix made of $\geq n/3$ columns from $H$ spans a space of dimension $\geq m-1$. We claim that it is the case if and only if, for any two distinct $\mathbf{x}$, $\mathbf{y}$ from $\mathbb{F}^m_2 \setminus \{\mathbf{0}^m\}$, and any subset $I \subseteq [n]$ of size $n/3$, it holds that $\mathbf{x}^\top\mathbf{c}_i \neq 0$ or $\mathbf{y}^\top\mathbf{c}_i \neq 0$ for some $i \in I$.
\par
To see this, first note $H_I$ the submatrix which columns are $(\mathbf{c}_i)_{i \in I}$. Then $H$ is $(m-1)$-good if and only if the dimension of $\inv{H_I}(\mathbf{0}^n)$ is at most $1$ for all $I \in [n]$ of size $n/3$. That is, for all such $I$, there must not be two different $\mathbf{x}$ and $\mathbf{y}$ distinct from $\mathbf{0}^m$ such that $H_I\mathbf{x} = H_I\mathbf{y} = \mathbf{0}^n$, or equivalently such that $\mathbf{x}^\top H_I^\top = \mathbf{y}^\top H_I^\top= \mathbf{0}^m$. Thus for some $i \in I$ there must be $\mathbf{x}^\top\mathbf{c}_i \neq 0$ or $\mathbf{y}^\top\mathbf{c}_i \neq 0$.

Fix any two distinct $\mathbf{x}$ and $\mathbf{y}$ different from $\mathbf{0}^m$. Let $Z_i$ be the random variable taking value $1$ if $\mathbf{c}_i$ is sampled such that $\mathbf{x}^\top\mathbf{c}_i = \mathbf{y}^\top\mathbf{c}_i = 0$, and $0$ otherwise. Note $Z = \sum_{i =1 }^n Z_i$. There is $\Prb{Z_i = 1} = 1/4$ and $\Ex{Z} = n/4$ by linearity of expectation. 
If the sampled matrix $H$ renders $Z \geq n/3$ true, then $H$ is not $(m-1)$-good from the previous claim. Using Hoeffding bounds we obtain
\begin{equation*}
\begin{aligned}
\Prb{Z \geq n/3} &= \Prb{Z - \Ex{Z} \geq n/12}
\\
&\leq \exp(-2(n/12)^2/n) = \exp(-n/72).
\end{aligned}
\end{equation*}
$H$ is $(m-1)$-good if and only if $Z < n/3$ holds for any $\mathbf{x}$ and $\mathbf{y}$. There are $\binom{2^m-1}{2} \leq 2^{2m}$ choices for these two vectors so by union bound $H$ is $(m-1)$-good with probability at least $1 - 2^{2m}e^{-n/72} = 1 - 2^{n/50}e^{-n/72} = 1 - 2^{-\Omega(n)}$.
\end{proof}
\end{lemma}

\begin{lemma}\textup{\cite{Duris04}}\label{lemma:size_perfect_rectangle_appendix}
Let $f$ be the characteristic function of a linear code of length $n$ characterized by the $s$-good matrix $H$. Let $r$ be a balanced rectangle such that $r \leq$~$f$. Then $\card{\inv{r}(1)} \leq 2^{n-2s}$.
\begin{proof}
Let $X =  \{x_1, \dots, x_n\}$ be the variables of $f$ and $(X_1, X_2)$ be the partition of $X$ for $r$ and note $r = \rho_1 \wedge \rho_2$. Index the columns of $H$ with the Boolean variables ($x_1$ for column 1, $x_2$ for column 2, and so on) and denote $H_1$ and $H_2$ the submatrices obtained keeping the columns indexed in $X_1$ and $X_2$ respectively. $H$ is $s$-good and the partition $(X_1, X_2)$ is balanced so $\rk{H_1} \geq s$ and $\rk{H_2} \geq s$.\par 
Take $\mathbf{x}$ a model of $r$ and denote $\mathbf{x}_1$ and $\mathbf{x}_2$ its restriction to $X_1$ and $X_2$ respectively. $\mathbf{x}_1$ and $\mathbf{x}_2$ are models of $\rho_1$ and $\rho_2$ respectively. Since $r \leq f$, $\mathbf{x}$ is a model of $f$ and it holds that $H\mathbf{x} = H_1\mathbf{x}_1 + H_2\mathbf{x}_2 = \mathbf{0}^m$. 
Note $\mathbf{w} = H_2\mathbf{x}_2$, then any $\mathbf{x}'_1$ model of $\rho_1$ has to satisfy $H_1 \mathbf{x}'_1 = \mathbf{w}$ (otherwise $r \nleq f$). Consequently there are $2^{\card{X_1} - \rk{H_1}} \leq 2^{\card{X_1} - s} $ models of $\rho_1$. Symmetrically $\rho_2$ has $\leq 2^{\card{X_2} - s}$ models. So $r$ has fewer than $2^{\card{X_1} + \card{X_2} - 2s} = 2^{n-2s}$ models.
\end{proof}
\end{lemma}

\setcounter{theorem}{2}

\begin{theorem}\label{theorem:large_dDNNF_linear_code_appendix}
There exists a class of linear codes $\mathcal{C}$ such that, for any code from $\mathcal{C}$ of length $n$, any {\dDNNF} encoding its characteristic function has size $2^{\Omega(n)}$.
\begin{proof}
Let $m = n/100$. Lemma~\ref{lemma:good_matrices_appendix} ensures the existence of $(m-1)$-good matrices in $\mathbb{F}^{m \times n}_2$ for $n$ large enough. Let $\mathcal{C}$ be the class of linear codes characterized by these matrices. Choose a code in $\mathcal{C}$ for the $(m-1)$-good matrix $H$ and denote $f$ its characteristic function, it has $2^{n - \rk{H}} \geq 2^{n - m}$ models. Let $\bigvee_{r \in R} r$ be a disjoint balanced rectangle cover of $f$. It holds that $\card{\inv{f}(1)} = \sum_{r \in R} \card{\inv{r}(1)}$ and we know from Lemma~\ref{lemma:size_perfect_rectangle_appendix} that any $r$ has $\leq 2^{n - 2s} = 4 \times 2^{n - 2m}$ models so $\card{\inv{f}(1)} \leq 4\card{R} \times 2^{n-2m}$. Using the lower bound on the number of models of $f$ we obtain $\card{R} \geq \frac{1}{4}2^{m}$. Applying Theorem~\ref{theorem:DNNF_size_rect_cover_size} finishes the proof.
\end{proof}
\end{theorem}

\setcounter{claim}{0}
\begin{claim}
Let $r = \rho_1 \wedge \rho_2$ be a rectangle w.r.t.~the partition $(X_1, X_2)$ and denote $(A,B) = \mathcal{C}_f(\inv{\rho_1}(1), \inv{\rho_2}(1))$. Then the rectangle $\mathbb{1}_A \wedge \mathbb{1}_B$ is a core rectangle of $r$ w.r.t.~$f$.
\end{claim}
\begin{proof}
The rectangle $r_0 = \mathbb{1}_A \wedge \mathbb{1}_B$ is defined w.r.t.~the same partition as $r$. We know justify that it is core rectangle for $f$, as defined in Definition~\ref{definition:core_rectangle}:
\begin{enumerate}
\item[$a)$] $A \subseteq \inv{\rho_1}(1)$ and $B \subseteq \inv{\rho_2}(1)$ so $r_0 \leq r$ and all assignments from $A \times B$ are models of $f$ so $r_0 \leq f$. 
\item[$b)$] Assume $r_0$ is not maximal, that is, there exist $A' \subseteq \inv{\rho_1}(1)$ and $B' \subseteq \inv{\rho_2}(1)$ such that $r' = \mathbb{1}_{A'} \wedge \mathbb{1}_{B'} \leq f$ and $\card{\inv{r'}(1)} > \card{\inv{r_0}(1)}$. Then $\card{A'}\card{B'} > \card{A}\card{B}$, which contradicts the properties of $\mathcal{C}_f$.\qedhere
\end{enumerate}
\end{proof}

\setcounter{claim}{2}

\begin{claim}
For any $i > 0$, all assignments from $F_i= (A_i \times \overB_i) \cup (\overA_i \times B_i)$ are false positives of $r$ on $f$. Furthermore for every $i\ne j$ we have $F_i \cap F_j = \emptyset$.
\begin{proof}
For the first part, it is clear from Claim~\ref{claim:false_positives} that assignments from $A_i \times \overB_i$ and $\overA_i \times B_i$ are false positives of $\mathbb{1}_{\overA_{i-1}} \wedge \mathbb{1}_{\overB_{i-1}}$ on $f$, and since $\mathbb{1}_{\overA_{i-1}} \wedge \mathbb{1}_{\overB_{i-1}} \leq r$, they are indeed false positives of $r$ on $f$. For the second part, let $j > i > 0$, 
$F_i=(A_i \times \overB_i) \cup (\overA_i \times B_i)$ and $F_j= (A_j \times \overB_j) \cup (\overA_j \times B_j)$ are disjoint because both $A_j$ and $\overA_j$ are disjoint from  $A_i$ and both $B_j$ and $\overB_j$ are disjoint from  $B_i$.
\end{proof}
\end{claim}

\begin{claim}
The function $\bigvee_{i = 1}^l (\mathbb{1}_{A_i} \wedge \mathbb{1}_{B_i})$ is a disjoint rectangle cover of $r \wedge f$. Furthermore, if $r$ is balanced, so are the rectangles from $\bigvee_{i = 1}^l (\mathbb{1}_{A_i} \wedge \mathbb{1}_{B_i})$.
\begin{proof}
By construction, the functions $(\mathbb{1}_{A_i} \wedge \mathbb{1}_{B_i})$ are rectangles with respect to the same partition as $r$. So if $r$ is balanced, so are these rectangles.
\par
For all $i$ there is $(A_i \times B_i) \subseteq \inv{r}(1)$, so $\bigvee_{i = 1}^l (\mathbb{1}_{A_i} \wedge \mathbb{1}_{B_i}) \leq r$. And by definition of $\mathcal{C}_f$, assignments from $A_i \times B_i$ are models of $f$, so $\bigvee_{i = 1}^l (\mathbb{1}_{A_i} \wedge \mathbb{1}_{B_i}) \leq r \wedge f$.
\par
To prove equality, assume that there exists $\mathbf{x}$ a model of~$r$ and $f$ that is not a model of $\bigvee_{i = 1}^l (\mathbb{1}_{A_i} \wedge \mathbb{1}_{B_i})$, that is, $\mathbf{x}$ does not belong to any $A_i \times B_i$ for $i > 0$. Then by Claim 3, $\mathbf{x}$ must be in $\overA_l \times \overB_l$ (figure~\ref{fig:iterative} may help seeing this), but since $\overA_l \times \overB_l$ contains no models of $f$, this contradicts our assumption.
\par
This proves that $\bigvee_{i = 1}^l (\mathbb{1}_{A_i} \wedge \mathbb{1}_{B_i})$ is a rectangle cover of $r \wedge f$. The only thing left to prove is that the rectangles are disjoint. To see this, it is sufficient to observe that, for all $i > 1$, $A_i \subseteq \overA_{i-1}$ which is disjoint from $A_{i-1}$ and $B_i \subseteq \overB_{i-1}$ which is disjoint from $B_{i-1}$.
\end{proof}
\end{claim}

\subsection*{Missing Proofs of Section~\ref{sec:weak_approximation}}
\label{appendix:bilinear_form}

\setcounter{theorem}{1}

\begin{theorem}\label{theorem:thm_weak_approx}
Let $0 \leq \varepsilon < 1/2$, there is a class of Boolean functions $\mathcal{C}$ such that, for any $f \in \mathcal{C}$ on $n$ variables, any {\dDNNF} encoding a weak $\varepsilon$-approximation of $f$ w.r.t.~$\mathcal{U}$ has size $2^{\Omega(n)}$.
\end{theorem}

Theorem~\ref{theorem:thm_weak_approx} is essentially proved in~\cite{BolligSW02} for a class of Boolean bilinear forms. 
\par
A function $f : \mathbb{F}_2^n \times \mathbb{F}_2^n \rightarrow \mathbb{F}_2$ is a \emph{bilinear form} if it is linear in each of its two arguments. Every bilinear form is characterized by a matrix $A$ from $\mathbb{F}_2^{n \times n}$ by the relation $f(\mathbf{x},\mathbf{y}) = \mathbf{x}^\top A \mathbf{y}$. 
Bilinear forms can be seen as Boolean functions from $\ZO^{2n}$ to $\ZO$, yet we find convenient to keep the notation $f(\mathbf{x},\mathbf{y})$.
\par The authors of~\cite{BolligSW02} find a class of bilinear forms which discrepancy with respect to combinatorial rectangles are small enough to apply~\ref{lemma:bound_size_cover_weak_approximations}. For the reader convenience we rewrite the proof in this appendix almost as is, the only statement we do not give proof of is the following lemma due to Ajtai. It states that for $n$ large enough, there exist matrices with a lower bound on the rank of any large enough submatrix.

\begin{lemma}[Ajtai Lemma~\cite{Ajtai05}]\label{lemma:ajtai}
Take $0 < \delta \leq 1/2$ such that $\delta \log(1/\delta)^2 \leq 2^{-16}$. There exist exponentially many matrices of $\mathbb{F}^{n \times n}_2$ for which each square submatrix of size at least $\delta n \times \delta n$ has rank at least $\delta'n$, where $\delta' = \delta/(256 \log(1/\delta))^2$.
\end{lemma}

One can prove Theorem~\ref{theorem:thm_weak_approx} for the class of bilinear forms characterized by the matrices describe by Ajtai's lemma. For the rest of the appendix note $\mathcal{C}$ this class of functions and define $\delta$ and $\delta'$ as in the lemma.

\begin{claim}\label{claim:number_models_bilinear_form}
For $f$ a bilinear form in $\mathcal{C}$ on $2n$ variables, there is $\card{\inv{f}(1)} \geq 2^{2n-1}(1 - 2^{-\delta'n})$.
\begin{proof}
A bilinear form characterized by a $n \times n$ matrix $M$ has $2^{2n-1}(1 - 2^{-\rk{M}})$ models (there are $2^{n} - 2^{n - \rk{M}}$ vectors $\mathbf{y}$ such that $M\mathbf{y} \neq \mathbf{0}^n$ and for each such $M\mathbf{y}$ there are $2^{n-1}$ vectors $\mathbf{x}$ such that $\mathbf{x}^\top M\mathbf{y} \neq 0$). When $M$ characterizes $f \in \mathcal{C}$, Ajtai's lemma tells us that any $\delta n \times \delta n$ submatrix of $M$ has rank at least $\delta'n$, thus $\rk{M} \geq \delta'n$ and the claim holds.
\end{proof}
\end{claim} 

\noindent The next claim is technical and is proved later in the appendix.

\begin{claim}\label{claim:discrepancy_ajtai_bilinear_form}
For $f$ a bilinear form in $\mathcal{C}$ on $2n$ variables, and $r$ a combinatorial rectangle on the same variables as $f$, there is $\Disc{f}{r} \leq 2^{-\delta'n/2}/4$.
\end{claim} 

We now prove Theorem~\ref{theorem:thm_weak_approx}.

\begin{proof}[Proof of Theorem~\ref{theorem:thm_weak_approx}]
Using Lemma~\ref{lemma:bound_size_cover_weak_approximations} and Claim~\ref{claim:discrepancy_ajtai_bilinear_form}, we find that for $f$ a bilinear form in $\mathcal{C}$ on $2n$ variables, if $\tilde{f}$ weakly $\varepsilon$-approximates $f$ when truth assignments are uniformly sampled, then any disjoint balanced rectangle cover of $\tilde{f}$ has size at least 
$$4\times\frac{\card{\inv{f}(1)}-\varepsilon 2^{2n}}{2^{2n}2^{-\delta'n/2}}.$$
Using Claim~\ref{claim:number_models_bilinear_form}, this is greater than 
$$4 \times \frac{\frac{1}{2}(1-2^{\delta'n})-\varepsilon}{2^{-\delta'n/2}} = 2^{\delta'n/2}(2(1-2^{\delta'n})-4\varepsilon) = 2^{\Omega(n)}.$$
So ones need exponentially many balanced disjoint rectangles to cover $\tilde{f}$. Theorem~\ref{theorem:DNNF_size_rect_cover_size} allows us to conclude.
\end{proof}

The only think left is to prove that Claim~\ref{claim:discrepancy_ajtai_bilinear_form} holds. The proof is slightly technical and can be found in~\cite{BolligSW02}. It requires four intermediate lemmas that we rewrite here for the sake of completeness. The first step is the following lemma giving an upper bound on the discrepancy when the rectangle's partition agrees with the two sets of $n$ variables used by the bilinear form.

\begin{lemma}\label{lemma:bilinear_discrepancy}
Let $f : X \times Y \rightarrow \ZO$ be a bilinear form characterised by the matrix $A$, and let $r$ be a rectangle over $X \cup Y$ w.r.t.~the partition $(X,Y)$, then
$$
\Disc{f}{r} \leq 2^{-\rk{A}/2}.
$$ 
\begin{proof}
Let $g : X \times Y \rightarrow \{-1,+1\}$ be defined by $g(\mathbf{x},\mathbf{y}) = 2f(\mathbf{x},\mathbf{y}) - 1$ and $r(\mathbf{x},\mathbf{y}) = \rho_1(\mathbf{x})\rho_2(\mathbf{y})$. Recall the definition of discrepance $\Disc{f}{r} = \bcard{ \card{\inv{r}(1) \cap \inv{f}(1)} - \card{\inv{r}(1) \cap \inv{f}(0)}}/2^{2n}$ and observe that it is equal to $\bcard{ \Prbl{\mathbf{x},\mathbf{y}}{r(\mathbf{x},\mathbf{y})g(\mathbf{x},\mathbf{y}) = 1} - \Prbl{\mathbf{x},\mathbf{y}}{r(\mathbf{x},\mathbf{y})g(\mathbf{x},\mathbf{y}) = -1}}$, where $\Prbl{\mathbf{x},\mathbf{y}}{\cdot}$ is the probability measure where $\mathbf{x}$ and $\mathbf{y}$ are sampled uniformly at random. Switching to the expectation:
\begin{equation*}
\begin{aligned}
\Disc{f}{r} 
&= \big\vert \Exl{\mathbf{x},\mathbf{y}}{g(\mathbf{x},\mathbf{y})\rho_1(\mathbf{x})\rho_2(\mathbf{y})} \big\vert  \\
&= \big\vert \Exl{\mathbf{x}}{\Exl{\mathbf{y}}{g(\mathbf{x},\mathbf{y})\rho_2(\mathbf{y})}\rho_1(\mathbf{x})} \big\vert \\
&\leq \textup{E}_\mathbf{x} \bcard{\Exl{\mathbf{y}}{g(\mathbf{x},\mathbf{y})\rho_2(\mathbf{y})}} \\
&\leq \left( \Exl{\mathbf{x}}{\Exl{\mathbf{y}}{g(\mathbf{x},\mathbf{y})\rho_2(\mathbf{y})}^2} \right)^{1/2}
\end{aligned}
\end{equation*}
Where the last inequality is the Cauchy-Schwarz bound. Let be $\mathbf{y}'$ an i.i.d. copy of $\mathbf{y}$
\begin{equation*}
\begin{aligned}
&\Exl{\mathbf{x}}{\Exl{\mathbf{y}}{g(\mathbf{x},\mathbf{y})\rho_2(\mathbf{y})}^2} \\
&= \Exl{\mathbf{x}}{\Exl{\mathbf{y}}{g(\mathbf{x},\mathbf{y})\rho_2(\mathbf{y})}\Exl{\mathbf{y}'}{g(\mathbf{x},\mathbf{y}')\rho_2(\mathbf{y}')}} \\
&= \Exl{\mathbf{y},\mathbf{y}'}{\Exl{\mathbf{x}}{g(\mathbf{x},\mathbf{y})g(\mathbf{x},\mathbf{y}')}\rho_2(\mathbf{y})\rho_2(\mathbf{y}')}
\end{aligned}
\end{equation*}
Now we study $\Exl{\mathbf{x}}{g(\mathbf{x},\mathbf{y})g(\mathbf{x},\mathbf{y}')}$ depending on $\mathbf{y}$ and $\mathbf{y}'$. When $A\mathbf{y} = A\mathbf{y}'$, there is $g(\mathbf{x},\mathbf{y})g(\mathbf{x},\mathbf{y}') = 1$, a fortiori $\Exl{\mathbf{x}}{g(\mathbf{x},\mathbf{y})g(\mathbf{x},\mathbf{y}')} = 1$. When $A\mathbf{y} \neq A\mathbf{y}'$, say $A\mathbf{y} = \mathbf{u}$ and $A\mathbf{y}'= \mathbf{v}$, then 
\begin{equation*}
\begin{aligned}
&\Exl{\mathbf{x}}{g(\mathbf{x},\mathbf{y})g(\mathbf{x},\mathbf{y}')}
\\ &= \Prbl{\mathbf{x}}{\mathbf{x}^\top \mathbf{u} = \mathbf{x}^\top \mathbf{v}} - \Prbl{\mathbf{x}}{\mathbf{x}^\top \mathbf{u} \neq \mathbf{x}^\top \mathbf{v}} 
\\ &= \Prbl{\mathbf{x}}{\mathbf{x}^\top (\mathbf{u}+\mathbf{v}) = 0} - \Prbl{\mathbf{x}}{\mathbf{x}^\top (\mathbf{u}+\mathbf{v}) = 1} = 0
\end{aligned}
\end{equation*}
We conclude that $\Exl{\mathbf{x}}{g(\mathbf{x},\mathbf{y})g(\mathbf{x},\mathbf{y}')} = \mathbb{1}[A\mathbf{y} = A\mathbf{y}']$ and $\Exl{\mathbf{x}}{\Exl{\mathbf{y}}{g(\mathbf{x},\mathbf{y})\rho_2(\mathbf{y})}^2}  \leq \Prbl{\mathbf{y},\mathbf{y}'}{A\mathbf{y} = A\mathbf{y}'} = \Prbl{\mathbf{y}}{A\mathbf{y} = \mathbf{0}^n}$. Combine this result with the bound for $\Disc{f}{r}$ to obtain
$$
\Disc{f}{r} \leq \sqrt{\Prbl{\mathbf{y}}{A\mathbf{y} = \mathbf{0}^n}} = 2^{-\rk{A}/2}
$$
\end{proof}
\end{lemma}

Lemma~\ref{lemma:bilinear_discrepancy} requires the partition of the rectangle to be $(X,Y)$. In general, there is no obligation for such a constraint, so the lemma has little interest alone but it can be used after \emph{conditioning}. Let $\mathbf{a}$ be an assignment of some subset of variables $S \subseteq X \cup Y$. We say that we \emph{condition $f$ on $\mathbf{a}$} when we fix inputs from $S$ as in $\mathbf{a}$ and look at the function on the $2n-\card{S}$ variables left. We note $f_\mathbf{a} : \ZO^{2n-\card{S}} \rightarrow \ZO$ the Boolean function resulting from conditioning $f$ on $\mathbf{a}$. Conditioning a rectangle $r$ on $\mathbf{a}$ gives another rectangle $r_\mathbf{a}$. We adopt the notations $S_X = S \cap X$ and $S_Y = S \cap Y$.

\begin{lemma}\textup{\cite{BolligSW02}}\label{lemma:subrectangles}
Let $r$ be a balanced rectangle over $X \cup Y$. For any $\delta \leq 2/3$ there exist $S \subseteq X \cup Y$ such that 
\begin{itemize}[leftmargin=*]
\item[\text{ }i.] both $S_X$ and $S_Y$ have size $\delta n$,
\item[ii.] for any assignment $\mathbf{a}$ of $S$, the function $r_\mathbf{a}$ resulting from conditioning $r$ on $\mathbf{a}$ is a rectangle on $(X \cup Y) \setminus S$ with respect to the partition $(S_X,S_Y)$.
\end{itemize}
\begin{proof}
Let $(X_1 \cup Y_1, X_2 \cup Y_2)$ be the partition of $r := \rho_1 \wedge \rho_2$, where $X = X_1 \cup X_2$ and $Y = Y_1 \cup Y_2$. Because $r$ is balanced it holds that $\bcard{\card{X_1 \cup Y_1} - \card{X_2 \cup Y_2}} \leq 2n/3$. Assume without loss of generality that $\card{X_1} \leq 2n/3$ and $\card{Y_2} \leq 2n/3$. Consider $C \subseteq X$ and $R \subseteq Y$ such that $X_2 \subseteq X \setminus C$ and $Y_1 \subseteq Y \setminus R$. We can find such $C$ and $R$ of size $\delta n$ for any $\delta \leq 2/3$. Let $S = C \cup R$ (therefore $S_X = C$ and $S_Y = R$). Conditioning $r$ on any assignment $\mathbf{a}$ of $S$ gives $r_\mathbf{a} = (\rho_1)_\mathbf{a} \wedge (\rho_2)_\mathbf{a}$ where $(\rho_1)_\mathbf{a}$ and $(\rho_2)_\mathbf{a}$ are defined on $S_X$ and $S_Y$ respectively. So $r_\mathbf{a}$ is a rectangle with respect to the partition $(S_X,S_Y)$.
\end{proof}
\end{lemma}

Now the partition of $r_\mathbf{a}$ matches the input spaces of $f_\mathbf{a}$. The next lemma states that if the discrepancies of $f_\mathbf{a}$ with respect to $r_\mathbf{a}$ for every $\mathbf{a}$ share a common upper bound, then the same upper bound holds for the discrepancy of $f$ with respect to $r$ .

\begin{lemma}\textup{\cite{BolligSW02}}\label{lemma:partialDiscr}
Let $f : X \times Y \rightarrow \ZO$ and let $r$ be a balanced rectangle over $X \cup Y$. For some given $0 < \delta < 2/3$, let $S_X$ and $S_Y$ be the subsets given by Lemma~\ref{lemma:subrectangles}. If for all assignments $\mathbf{a}$ of $S := S_X \cup S_Y$ there is $\Disc{f_\mathbf{a}}{r_\mathbf{a}} \leq \beta$, then $\Disc{f}{r} \leq \beta$.
\begin{proof}
Let $Z = X \cup Y$. We start from the probabilistic definition of the discrepancy 
\begin{equation*}
\begin{aligned}
\Disc{f}{r} = &\text{ }\big\vert \Prbl{\mathbf{z}}{f(\mathbf{z}) = 1 \text{ and } r(\mathbf{z}) = 1} \\
			  &- \Prbl{\mathbf{z}}{f(\mathbf{z}) = 0 \text{ and } r(\mathbf{z}) = 1} \big\vert 
\end{aligned}
\end{equation*}
Use Lemma~\ref{lemma:subrectangles} to get $S$ and denote $\mathbf{z}'$ the restriction of $\mathbf{z}$ to $Z \setminus S$. Let $p_\mathbf{a}$ be the probability to sample the assignment $\mathbf{a}$ of $S$ uniformly at random.
\begin{equation*}
\begin{aligned}
\Disc{f}{r} 
&= \big\vert \sum_\mathbf{a} \Prbl{\mathbf{z}'}{f_\mathbf{a}(\mathbf{z}') = r_\mathbf{a}(\mathbf{z}') = 1}p_\mathbf{a} \\
& \text{ }\text{ }\text{ } - \sum_\mathbf{a} \Prbl{\mathbf{z}'}{f_\mathbf{a}(\mathbf{z}') \neq r_\mathbf{a}(\mathbf{z}') = 1}p_\mathbf{a} \big\vert \\
&\leq \sum_\mathbf{a} \Disc{f_\mathbf{a}}{r_\mathbf{a}}p_\mathbf{a} \leq \beta
\end{aligned}
\end{equation*}
\end{proof}
\end{lemma}

One last problem is that after conditioning $f$ on $\mathbf{a}$, the resulting function $f_\mathbf{a}$ is not necessarily a bilinear form. However this last lemma shows that we can lift it to a bilinear form introducing additional variables.

\begin{lemma}\textup{\cite{BolligSW02}}\label{lemma:bilinearExtension}
Let $f : X \times Y \rightarrow \ZO$ be a bilinear form on $2n$ variables characterised by the $n \times n$ matrix $H$. 
\begin{itemize}[leftmargin=*, topsep=0.1cm]
\item[$\bullet$] For $0 < \delta < 1$, let $C$ be a subset of $X$ and let $R$ be a subset of $Y$, such that $\card{C} = \card{R} = \delta n$. 
\item[$\bullet$] Consider an arbitrary assignment $\mathbf{a}$ of $(X \cup Y) \setminus (C \cup R)$ and condition $f$ on $\mathbf{a}$.
\item[$\bullet$]Let $A$ be the $\delta n \small{\times} \delta n$ submatrix of $H$ obtained taking rows indexed in $R$ and columns indexed in $C$.
\end{itemize}
Given two additional variables $e_1$ and $e_2$, let $\widehat{C} = \{e_1\} \cup C$ and $\widehat{R} = \{e_2\} \cup R $. There is a bilinear form $\widehat{f}_\mathbf{a} : \widehat{C} \times \widehat{R} \rightarrow \ZO$ such that $f_\mathbf{a}$ results from conditioning $\widehat{f}_\mathbf{a}$ on $(e_1 = 1$, $e_2 = 1)$. Furthermore the matrix $\widehat{A}$ characterising $\widehat{f}_\mathbf{a}$ has rank $\rk{\widehat{A}} \geq \rk{A}$.
\begin{proof}
By definition for $\mathbf{x},\mathbf{y} \in X \times Y$, $f(\mathbf{x},\mathbf{y}) = \mathbf{x}^\top H \mathbf{y}$. Write $\mathbf{x} = (\mathbf{x}_C\text{ }\mathbf{x}_{\overline{C}})$ (resp. $\mathbf{y} = (\mathbf{y}_R \text{ }\mathbf{y}_{\overline{R}})$) to distinguish the entries of $\mathbf{x}$ (resp. $\mathbf{y}$) that belongs to $C$ (resp. $R$). By conditioning $f$ on $\mathbf{a}$, we fix the values of $\mathbf{x}_{\overline{C}}$ and $\mathbf{y}_{\overline{R}}$. We can write $f_\mathbf{a}(\mathbf{x}_C,\mathbf{y}_R) = \mathbf{x}_C^\top A \mathbf{y}_R + \mathbf{x}_C^\top \mathbf{v} + \mathbf{u}^\top \mathbf{y}_R + \lambda$ for some vectors $\mathbf{u}$, $\mathbf{v}$ and a value $\lambda$, all three depending only of the variables fixed in $\mathbf{a}$. Define $\widehat{f}_\mathbf{a} : \ZO^{2\delta n + 2} \rightarrow \ZO$ by $\widehat{f}_\mathbf{a}(\mathbf{x}', \mathbf{y}') = \mathbf{x}'^\top \widehat{A} \mathbf{y}'$ where 
\[
\widehat{A} =
\left( 
\begin{array}{c|ccc}
   \lambda &  & u^\top &  \\ \hline \\[-9pt]
     & \multicolumn{3}{c}{\multirow{3}{*}{\textit{A}}} \\
   v & \\
     &
\end{array}
\right)
\]
$\widehat{f}_\mathbf{a}$ is a bilinear form (while $f_\mathbf{a}$ was not necessarily). Calling $e_1$ (resp. $e_2$) the first variable of $\mathbf{x}'$ (resp. $\mathbf{y}'$) we do find that when $\mathbf{x}' = (1 \text{ }\mathbf{x}_C)$ and $\mathbf{y}' = (1\text{ }\mathbf{y}_R)$ there is $\widehat{f}_\mathbf{a}(\mathbf{x}',\mathbf{y}') = f_\mathbf{a}(\mathbf{x}_C, \mathbf{y}_R)$. Furthermore it is immediate that $\rk{\widehat{A}} \geq \rk{A}$.
\end{proof}
\end{lemma}

We can finally prove Claim~\ref{claim:discrepancy_ajtai_bilinear_form}.
\begin{proof}[Proof of Claim~\ref{claim:discrepancy_ajtai_bilinear_form}]
Recall that $f$ is a bilinear form in $\mathcal{C}$ characterised by a $n \times n$ matrix $M$ defined as in Ajtai's Lemma. Let $\tilde{f}$ be an weak $\varepsilon$-approximation of $f$, and let $\bigvee_{k = 1}^K r_k$ be a disjoint balanced rectangle cover of $\tilde{f}$. Let $r$ be any rectangle from this cover.\\
Let $S$ be as given by Lemma~\ref{lemma:subrectangles} for the rectangle $r$ and the value $\delta$ from Ajtai's Lemma. Both $S_X$ and $S_Y$ $H$ have size $\delta n$ so the submatrix  of $M$ obtained choosing columns indexed from $S_X$ and rows indexed from $S_Y$ has rank $\geq \delta' n$. Let $\mathbf{a}$ be an arbitrary assignment of $(X \cup Y) \setminus S$, denote $f_\mathbf{a}$ the function $f$ conditioned on $\mathbf{a}$ and $r_\mathbf{a}$ the rectangle $r$ conditioned on $\mathbf{a}$. Lemma~\ref{lemma:subrectangles} ensures that $r_\mathbf{a}$ is a rectangle for the partition $(S_X,S_Y)$.\\
$f_\mathbf{a}$ is not necessarily bilinear but through Lemma~\ref{lemma:bilinearExtension} we can work with its bilinear extension $\hat{f}_\mathbf{a} : \hat{S}_X \times  \hat{S}_Y \rightarrow \ZO$ (where $\hat{S}_X = \{e_1\} \cup S_X$ and  $\hat{S}_Y = \{e_2\} \cup S_Y$) characterised by the $(\delta n+1) \times (\delta n+1)$ matrix $\hat{A}$ of rank $\geq \delta'n$. Now the rectangle $r_\mathbf{a}$ is not defined on the same input space as $\hat{f}_\mathbf{a}$ but we can consider the extension $\hat{r}_\mathbf{a} : \hat{S}_X \cup \hat{S}_Y \rightarrow \ZO$ defined as $\hat{r}_\mathbf{a}(e_1, \mathbf{x}, e_2, \mathbf{y}) = r_\mathbf{a}(\mathbf{x}, \mathbf{y})$ when $e_1 = e_2 = 1$ and $0$ otherwise. $\hat{r}_\mathbf{a}$ is a rectangle with respect to the partition $( \hat{S}_X , \hat{S}_Y )$ so by Lemma~\ref{lemma:bilinear_discrepancy}, 
$$
\Disc{\hat{f}_\mathbf{a}}{\hat{r}_\mathbf{a}} \leq 2^{-\rk{A}/2} \leq 2^{-\delta'n/2}
$$
Observe that $\card{\inv{\hat{r}}_\mathbf{a}(1) \cap \inv{\hat{f}}_\mathbf{a}(1)} = \card{\inv{r}_\mathbf{a}(1) \cap \inv{f}_\mathbf{a}(1)}$ and $\card{\inv{\hat{r}}_\mathbf{a}(1) \cap \inv{\hat{f}}_\mathbf{a}(0)} = \card{\inv{r}_\mathbf{a}(1) \cap \inv{f}_\mathbf{a}(0)}$ so that 
$$
\Disc{\hat{f}_\mathbf{a}}{\hat{r}_\mathbf{a}} = \frac{1}{4}\Disc{f_\mathbf{a}}{r_\mathbf{a}}
$$
The assignment $\mathbf{a}$ has been chosen arbitrarily so $\Disc{f_\mathbf{a}}{r_\mathbf{a}} \leq 2^{-\delta'n/2}/4$ holds for any $\mathbf{a}$. A fortiori, it holds from Lemma~\ref{lemma:partialDiscr} that $\Disc{f}{r} \leq 2^{-\delta'n/2}/4$.
\end{proof}

\end{document}